\documentclass[10pt, conference]{IEEEtran}
\IEEEoverridecommandlockouts
\usepackage{cite}
\usepackage{amsmath,amssymb,amsfonts}
\usepackage{algorithmic}
\usepackage{graphicx}
\usepackage{textcomp}
\usepackage{xcolor}

\usepackage{subcaption}
\usepackage{amsmath}
\usepackage{amsfonts}
\usepackage{algorithmic}
\usepackage{textcomp}
\usepackage{xcolor}
\usepackage{balance}
\usepackage{soul}
\soulregister\citep7
\usepackage{times}
\usepackage{latexsym}
\usepackage{graphicx}
\usepackage{xspace}
\usepackage{autobreak}
\usepackage{boxedminipage}
\usepackage{color}
\usepackage{CJK}
\usepackage{soul}
\usepackage{ulem}
\usepackage{fontawesome}
\usepackage{booktabs}
\usepackage{tcolorbox}
\usepackage{multirow}
\usepackage{xcolor}
\usepackage{soul}
\usepackage[strings]{underscore}

\usepackage{amsthm}

\newtheorem{theorem}{Proposition}[section]

\def\BibTeX{{\rm B\kern-.05em{\sc i\kern-.025em b}\kern-.08em
    T\kern-.1667em\lower.7ex\hbox{E}\kern-.125emX}}

\newtcolorbox{answerbox}{
colback=white!20!gray, 
  colframe=black,
  fonttitle=\bfseries,
  before=\par\medskip\noindent,
  after=\par\medskip,
  coltitle=black,
  colbacktitle=white,
  colframe=white
}

\begin{document}

\title{Towards Better Fairness-Utility Trade-off: A Comprehensive Measurement-Based Reinforcement Learning Framework\\
}
\author{
\IEEEauthorblockN{Simiao Zhang}
\IEEEauthorblockA{East China Normal Universitty \\
smzhang@stu.ecnu.edu.cn}
\and
\IEEEauthorblockN{Jitao Bai}
\IEEEauthorblockA{Tianjin University \\
jitaobai_123@tju.edu.cn}
\and
\IEEEauthorblockN{Menghong Guan}
\IEEEauthorblockA{East China Normal University \\
mhguan@stu.ecnu.edu.cn}
\and
\IEEEauthorblockN{Yihao Huang}
\IEEEauthorblockA{Nanyang Technological University \\
huangyihao22@gmail.com}
\and
\IEEEauthorblockN{Yueling Zhang}
\IEEEauthorblockA{East China Normal University \\
ylzhang@sei.ecnu.edu.cn}
\and
\IEEEauthorblockN{Jun Sun}
\IEEEauthorblockA{Singapore Management University \\
junsun@smu.edu.sg}
\and
\IEEEauthorblockN{Geguang Pu}
\IEEEauthorblockA{East China Normal University \\
ggpu@sei.ecnu.edu.cn}
}

\maketitle
\newcommand{\ACC}{ACC} 
\newcommand{\AUC}{AUC}
\newcommand{\DI}{DI}
\newcommand{\SPD}{SPD}
\newcommand{\EOD}{EOD}
\newcommand{\AOD}{AOD}
\newcommand{\ERD}{ERD}
\newcommand{\LR}{LR}
\newcommand{\SVM}{SVM}
\newcommand{\MLP}{MLP}
\newcommand{\CFU}{$\mathrm{CFU}$}
\newcommand{\ROC}{$\mathrm{ROC}$}
\newcommand{\ADV}{$\mathrm{ADV}$}
\newcommand{\DIR}{$ \mathrm{DIR}$}
\newcommand{\GetFair}{$ \mathrm{GetFair}$}
\newcommand{\CNEW}{$ \mathrm{CFP_{new}}$} 
\newcommand{\CNOR}{$ \mathrm{CFP_{nor}}$} 
\newcommand{\winwin}{\textit{win-win}} 
\newcommand{\good}{\textit{good}} 
\newcommand{\inverted}{\textit{inverted}} 
\newcommand{\bad}{\textit{bad}} 
\newcommand{\loselose}{\textit{lose-lose}} 

\definecolor{yc}{RGB}{62,179,195}
\newcommand{\chen}[1]{{\color{yc}[yangchen: #1]}}
\newcommand{\chendelete}[1]{{\color{yellow}[delete: #1]}}
\newcommand{\yihao}[1]{{\color{red}[yihao: #1]}}
\newcommand{\yueling}[1]{{\color{magenta}[yueling: #1]}}

\begin{abstract}
Machine learning is widely used to make decisions with societal impact such as bank loan approving, criminal sentencing, and resume filtering. How to ensure its fairness while maintaining utility is a challenging but crucial issue. 
Fairness is a complex and context-dependent concept with over 70 different measurement metrics. 
Since existing regulations are often vague in terms of which metric to use and different organizations may prefer different fairness metrics, it is important to have means of improving fairness comprehensively.
Existing mitigation techniques often target at one specific fairness metric and have limitations in improving multiple notions of fairness simultaneously. 
In this work, we propose CFU (\underline{C}omprehensive \underline{F}airness-\underline{U}tility), a reinforcement learning-based framework, to efficiently improve the fairness-utility trade-off in machine learning classifiers.
A comprehensive measurement that can simultaneously consider multiple fairness notions as well as utility is established, and new metrics are proposed based on an in-depth analysis of the relationship between different fairness metrics. 
The reward function of CFU is constructed with comprehensive measurement and new metrics.
We conduct extensive experiments to evaluate CFU on 6 tasks, 3 machine learning models, and 15 fairness-utility measurements. The results demonstrate that CFU can improve the classifier on multiple fairness metrics without sacrificing its utility. It outperforms all state-of-the-art techniques and has witnessed a 37.5\% improvement on average.
\end{abstract}


\section{Introduction}
Machine learning (ML) techniques has been widely used in many areas, such as criminal sentencing \cite{lcifr}, medical imaging \cite{ricci2022addressing}, and credit risk evaluation \cite{KHANDANI20102767}. As a common task for ML software, classification is drawing special attention in helping people make decisions. Though the software usually has high accuracy in classification tasks, it has also raised public concerns about its fairness. For instance, it is found that a neural network trained to predict recidivism is more likely to label male or non-Caucasian with higher risks \cite{mattuMachineBias}. 
Moreover, in terms of neural networks trained for credit rating, males or the elderly are more likely to get better credit \cite{german}. Such decisions are unacceptable since they would bring moral condemnations or even legal issues. 
Therefore, it is critical to guarantee the fairness of these techniques. 

To address the issue of fair classification, several bias mitigation techniques have been developed \cite{cct11, cct15, cct20, cct25, cct26, cct28, cct29, cct30}. 
But in general, they are plagued by two challenges: the complex notions of fairness, and the trade-off between fairness and utility.
The first challenge for handling fairness issues lies in the conflicts of different fairness measurements \cite{impossible}.
It is described that there are at least 21 mathematical definitions on fairness in current literature \cite{aif360}, like Statistical Parity \cite{DI}, Equalized Odds \cite{AOD}, Equal Opportunity \cite{AOD}, etc. To quantitatively measure fairness, over 70 metrics \cite{aif360} have been proposed based on those definitions. 
Although they have provided an understanding of fairness from different perspectives, it is hard to combine these metrics together due to the inevitable conflicts among them \cite{cct5, cct31}. 
But a comprehensive set of metrics is usually necessary for evaluation of fairness \cite{cct6}, as no current legal frameworks have appointed which metric to use \cite{lcifr, KHANDANI20102767} and there is no one best metric relevant for all contexts \cite{aif360}.
Another issue concerns the trade-off between fairness and utility.
It is generally believed that such trade-off is inherent \cite{cct3}. That means an increase in fairness cannot be achieved without sacrificing accuracy \cite{intro1,cct31, intro2, fairea, intro3}.
An empirical study by Zhang and Sun \cite{cct31} has shown that many existing bias mitigation techniques that can improve fairness would always result in significant accuracy loss.
Furthermore, whether existing techniques are still effective when it comes to the trade-off across multiple metrics is also doubtful.
Unfortunately, solutions for the two challenges have rarely been reported in existing literature.
Even though some studies have tried to cover the challenges (e.g., Ref. \cite{getfair}), there is still a lack of generality.
Therefore, it is necessary to develop a generalized technique that can simultaneously improve multiple notions of fairness while balancing the trade-off between utility and fairness.
In this work, a comprehensive measurement is established through statistical methods, which can simultaneously consider multiple metrics with different scales and monotonicity.
We present both theoretical and empirical results to illustrate the relationship of some fairness metrics, based on which new independent metrics that are more effective are proposed.
With the newly developed comprehensive measurement and metrics, a reinforcement learning framework named \CFU~(\underline{C}omprehensive \underline{F}airness-\underline{U}tility) is proposed to improve
the fairness-utility trade-off of ML classifier.
Experimental investigation of \CFU~is conducted on 6 widely-used classification tasks and 3 ML algorithms, and it is compared with 4 state-of-the-art methods across 15 fairness-utility measurements (i.e., combinations of 5 fairness metrics and 3 utility metrics) by one recent benchmarking tool. 
Results show that \CFU~can enhance ML classifiers across 5 fairness metrics while maintaining acceptable utility on 3 utility metrics.
Moreover, with our new metrics, an improvement in the fairness-utility trade-off of \CFU~could be expected compared with the combinations of existing metrics. 
\CFU~can outperform all the compared methods across all ML algorithms and tasks, and has witnessed a 37.5\% improvement on average. 

In summary, we make the following contributions:
\begin{itemize}
\item We propose \CFU, a reinforcement learning framework, which can improve fairness-utility trade-off of ML classifiers across multiple metrics.
\item We establish a comprehensive measurement for multiple metrics regardless of their scale and monotonicity.
\item We reveal the coupling effects between different fairness metrics, and develop new metrics which allow us to improve multiple fairness metrics simultaneously.
\item We comprehensively evaluate \CFU~on 6 tasks and 3 ML algorithms, and compare it with 4 state-of-the-art methods. 
\CFU~has significantly improved fairness-utility trade-off across all tasks and models.
\end{itemize}

The rest of the paper is organized as follows. In Section \ref{chap:preli}, we review some background knowledge. The methodology of this work is presented in detail in Section \ref{chap:method}. In Section \ref{chap:exper}, we describe the experimental setup and research questions, and the results are discussed in Section \ref{chap:resul}. 
The conclusions are drawn in Section \ref{chap:concl}.

\section {Background}\label{chap:preli}

\subsection{Fairness Definitions}\label{sec:3.1.2}
We begin with introducing the relevant terminology from the field of machine learning fairness. 
$sensitive~attribute$ (e.g., race, sex, age, religion) is an input feature associated with a protected characteristic, which is application specific.
Based on the value of sensitive attributes, all individuals are divided into two groups, i.e. $privileged~groups$ (have advantages in getting favorable labels) and $unprivileged~groups$ (have disadvantages in getting favorable labels). 
Let $D=\{X, Y, Z\}$ be a dataset where $X$ is the training features, $Y$ is the binary classification label ($Y=1$ for favorable label and $Y=0$ for unfavorable label), and $Z$ is the binary sensitive attribute ($Z=1$ for privileged group and $Z=0$ for unprivileged group).
Define $\hat{Y}$ as the predicted label ($\hat{Y}$=1 for favorable label and $\hat{Y}$=0 for unfavorable label), and then some probability terms are given in Table~\ref{table1}. 
True positive rates for privileged group and unprivileged group are denoted by $TPR_p$ and $TPR_u$ respectively. Similarly, $FPR_p$ and $FNR_p$ represent the false positive rate and false negative rate for privileged group respectively, while $FPR_u$ and $FNR_u$ represent the false positive rate and false negative rate for unprivileged group respectively.

 \begin{table}[!ht]
 \caption{Definitions for different terms}
    \centering
    \begin{tabular}{ccccc}
    \hline
        \textbf{Terms} & \textbf{Definitions} \\ \hline
        $FPR_u$ & $P[\hat{Y}=1 \mid Y=0, Z=0]$ \\ 
        $FNR_u$ & $P[\hat{Y}=0 \mid Y=1, Z=0]$ \\ 
        $TPR_u$ & $P[\hat{Y}=1 \mid Y=1, Z=0]$ \\ 
        $FPR_p$ & $P[\hat{Y}=1 \mid Y=0, Z=1]$ \\ 
        $FNR_p$ & $P[\hat{Y}=0 \mid Y=1, Z=1]$ \\ 
        $TPR_p$ & $P[\hat{Y}=1 \mid Y=1, Z=1]$ \\ 
       \hline
    \end{tabular}
    \label{table1}
\end{table}

We proceed to briefly introduce 4 group-level fairness from both their statistical features and applications, along with the corresponding fairness metrics that are used to quantify the degree of fairness.
We focus on the following fairness notions as they are widely used in fairness literature \cite{cct6,aif360, whyhow, uclcom, ucltest, maat}.

\subsubsection{Statistical parity}
A classifier satisfies statistical parity if the proportion of individuals classified as the favorable labels is equal for two subgroups. That is,
$$P[\hat{Y}=1|Z=0]=P[\hat{Y}=1|Z=1]$$
Statistical parity is suitable for cases like employment or admissions, where there is a desire or legal requirement to ensure that individuals from different subgroups are employed or admitted in equal proportions.

For a given model, we can measure its degree of statistical parity using $\textit{disparate}$ $\textit{impact}$ (DI) \cite{DI} or $\textit{statistical}$ $\textit{parity}$ $\textit{difference}$ (SPD) \cite{SPD}.
\begin{equation}
DI = min(\frac{P[\hat{Y}=1|Z=0]}{P[\hat{Y}=1|Z=1]},\frac{P[\hat{Y}=1|Z=1]}{P[\hat{Y}=1|Z=0]})
\label{fair5}
\end{equation}
\begin{equation}
SPD = |P[\hat{Y}=1|Z=0]-P[\hat{Y}=1|Z=1]|
\label{fair1}
\end{equation}

\subsubsection{Equalized odds}
A classifier satisfies equalized odds if two subgroups have equal true positive rates and false positive rates. That is,
$$TPR_u = TPR_p~and~FPR_u = FPR_p$$
Equalized odds is satisfied provided that regardless of the race or gender of applicants, they have an equal chance of being admitted to one program if they are qualified, and an equal chance of being rejected if they are not qualified.

For a given model, we can measure its degree of equalized odds using $\textit{average odds difference}$ (AOD) \cite{AOD}.
\begin{equation}
AOD =\frac{1}{2}[(FPR_{u} - FPR_{p}) + (TPR_{u} - TPR_{p})]
\label{fair3}
\end{equation}

\subsubsection{Equal opportunity}
A classifier satisfies equal opportunity if the true positive rates are the same for the two subgroups. That is,
$$TPR_{u} = TPR_{p}$$
Equal opportunity aims to ensure people with similar abilities have an equal chance to achieve a desired outcome regardless of their attributes. 

For a given model, we can measure its degree of equal opportunity using $\textit{equal opportunity difference}$ (EOD) \cite{AOD}.
\begin{equation}
EOD = TPR_{u} - TPR_{p}
\label{fair2}
\end{equation}

\subsubsection{Error rate balance}
A classifier satisfies error rate balance if the false positive and false negative error rates are equal across groups. That is,
$$FPR_{u} = FPR_{p}~and~FNR_{u} = FNR_{p}$$
Error rate balance is important in situations where there is a need to ensure that the cost of errors is distributed fairly across different groups, such as in criminal justice or credit scoring. 

For a given model, we can measure its degree of error rate balance using $\textit{error rate difference}$ (ERD) \cite{ERD}.
\begin{equation}
ERD = (FPR_{u}+FNR_{u})-(FPR_{p}+FNR_{p})
\label{fair4}
\end{equation}


\subsection{Bias Mitigation Techniques} 
Bias mitigation techniques can be classified into three categories, including pre-processing techniques, in-processing techniques, and post-processing techniques \cite{cct7}.

Pre-processing techniques are developed based on the assumption that the root cause of unfairness comes from the biased features in training data~\cite{cct9}. Therefore, this family of techniques works on the dataset before training so that models can produce fairer predictions. Preprocessing of data could usually be achieved through the suppression of the sensitive attribute, massaging the dataset by changing class labels, and reweighing or resampling the data to remove discrimination without relabeling instances \cite{cct10}. A representative technique is reweighing \cite{cct11}, which assigns different weights to classes of input according to their degree of being favored \cite{cct6}. Others include modifying feature values \cite{cct12} or feature representations~\cite{cct13,cct14}. Semi-supervised learning framework in pre-processing phase has also been reported \cite{cct15}. However, modifying training data may have unexpected impacts on model fairness \cite{cct16}, and there is a huge controversy in the community on whether such action is acceptable.

In-processing techniques modify the ML model to mitigate the bias in predictions. One common technique is adversary training. Usually, adversary training is performed on model level \cite{cct17,cct18,cct19}, but that on individual neuron level has also been developed \cite{cct20}. Like the case in pre-processing techniques, supervised and semi-supervised learning are also applied in some in-processing techniques \cite{cct21,cct22}. Recent studies have proposed special designs for hyperparameters as another possible strategy for improving fairness in ML models \cite{cct23,cct24}. It is worth noting that some pre-processing methods have been integrated in in-processing techniques to form new in-processing ones \cite{cct25,cct26}.

Post-processing techniques modify the prediction results directly for less bias. They remove discrimination by adjusting some predictions of a classifier \cite{AOD}. Multiple post-processing techniques have been developed based on post-hoc disparity \cite{cct28}, Gaussian process \cite{cct29}, group-aware threshold adaptation \cite{cct30}, etc.

In addition to the techniques mentioned above, causality analysis is another way to improve model fairness \cite{cct31,cct32,cct33}. Several strategies like Pareto principle \cite{cct34,cct35} and mutual information \cite{cct36} have been explored in these techniques for better trade-off between fairness and performance.
\subsection{Reinforcement Learning for Fairness}
Fairness in reinforcement learning (RL) was first investigated in 2017 \cite{cct37}. Then the fairness of reinforcement frameworks in real-world applications has been widely studied \cite{cct38,cct39,cct40}. Deep reinforcement learning was also used for fairness testing in ML models \cite{cct41}. However, it should be noted that the studies at that time only cares about the fairness in reinforcement frameworks, rather than improving fairness through RL. It is not until 2021 that RL methods were applied to solve the problems concerning fairness in classification. In this study, Monte Carlo policy gradient method was developed to deal with non-differentiable constraints \cite{cct42}. After that, RL has been used to achieve fairness in recommendation \cite{cct43} and ranking problems \cite{cct44}. Sikdar et al. \cite{getfair} developed a reinforcement framework called \GetFair, which can get competitive accuracy fairness trade-offs across tested classifier models, datasets, and fairness metrics. 
It can also simultaneously improve statistical parity and equal opportunity by directly averaging the two fairness scores in the reward function. However, it cannot directly average more fairness metrics that have different monotonicity and value domains. Additionally, it  ignores the relationships between different fairness notions, and the performance is unclear when more fairness metrics are involved.


\begin{figure*}[htbp]
  \centering
  \includegraphics[width=1.0\textwidth]{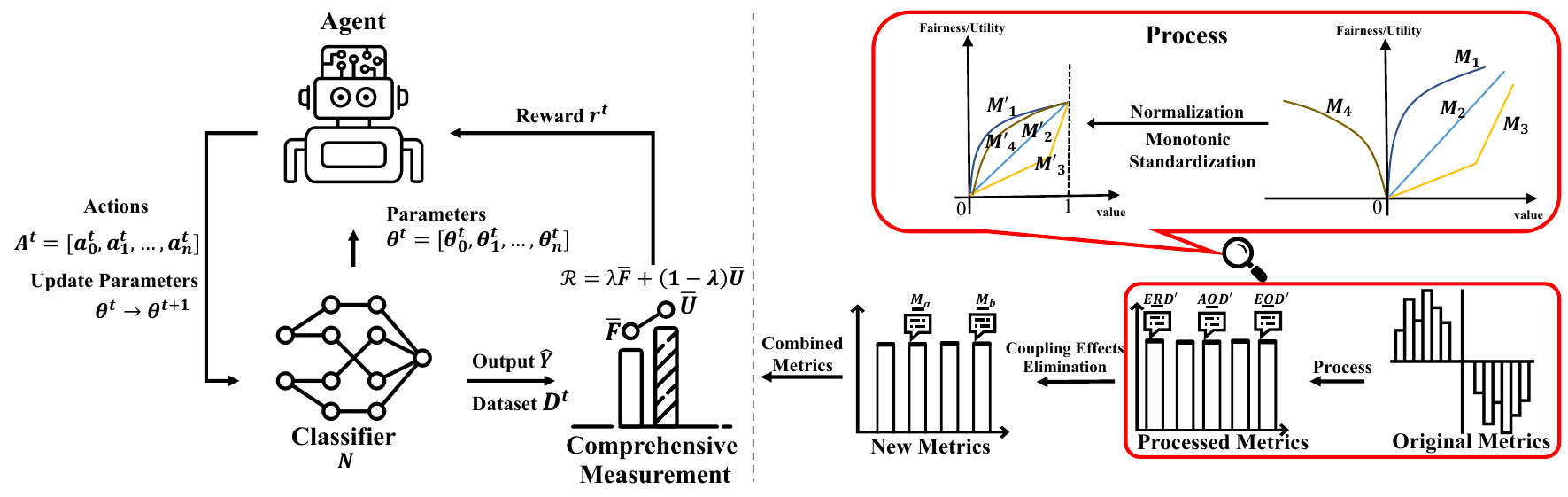}
  \caption{Overview of \CFU. 
The left part shows the training process. At step $t$, the agent inputs parameter $\theta^t$ and outputs the update directions
$A^t$.  $\theta^{t+1}$ is computed from $\theta^t$ by taking a step towards $A^t$. Then the classifier $N_{\theta^{t+1}}$ and the tuning dataset $D^{t}$ are used to calculate the reward $r^t$.
The right part shows the comprehensive measurement of the reward function from right to left. 
Original metrics are processed to the same monotonicity and scales. 
To eliminate coupling effects among metrics, we introduce new metrics $M_a$ and $M_b$.  The values of processed fairness (utility) metrics are averaged to obtain $\bar{F}$ ($\bar{U}$). 
}
  \label{fig:CFParchi}
\end{figure*}

\section{Methodology}\label{chap:method}
In this section, we first give an overview of CFU, then introduce the comprehensive measurement that combines multiple metrics in the reward function. 
Finally, we analyze the relationship of existing metrics and propose new metrics to enhance the performance of \CFU.

\subsection{CFU Framework}
The overall framework of \CFU ~is shown in Figure~\ref{fig:CFParchi}, in which the training process is shown on the left, and the construction of the reward function on the right. The training process is developed based on \GetFair \cite{getfair}, and the reward function is constructed based on comprehensive measurement introduced in Section~\ref{sec3.4}.

\textbf{Training process.} 
As a reinforcement learning framework, the structure of \CFU ~is also composed of the agent and environment, i.e., 
the classifier $N_{\theta}$ with parameters $\theta=[\theta_0, \theta_1,...,\theta_n]$  that we want to improve for better fairness-utility trade-off.
We adopt the standard notations \cite{FIGCPS} for the state $s^t$, action $a^t$, reward $r^t$, and policy $\pi$.
At the time step $t$, state is the parameters $\theta^t=[\theta^t_0, \theta^t_1,...,\theta^t_n]$ of the classifier,  
action $a^t_i \in \{-1, 1\}$ is the update direction of the $\theta^t_i$,
and reward $r^t$ is calculated by the reward function $\mathcal{R}$ defined in Eq~\ref{eqNext}.
Policy $\pi_{\phi}$ is a neural network called meta-optimizer with the parameters $\phi$, which takes in the parameter sequence $\theta^t$ and outputs an action sequence $A^t=[a^t_0, a^t_1, ..., a^t_n]$.

The detailed process for updating the parameters of the classifier is as follows. 
In the beginning, we get the classifier $N_{\theta^0}$ trained for utility.
At a time step $t$ in an episode, $\pi_{\phi}$ takes the parameters $\theta^t$ as input and generates the action $A^t$. 
The next parameter $\theta^{t+1}$ is computed from $\theta^t$ by taking a step towards the direction $A^t$ as Eq.~\ref{equpdate},
\begin{equation}
\label{equpdate}
    \theta^{t+1} = \theta^t + A^t * lr * c(t)
\end{equation}
where $lr$ is a predetermined step size similar to the learning rate used in gradient descent methods, and $c(t)$ is an optional scaling parameter that reduces the step size when approaching the optimum.
The classifier with the parameter $\theta^{t+1}$ is then used to calculate the reward on tuning dataset $D_t \subset D$. 
This process is repeated to generate an episode, and a trajectory is obtained: 
$$traj = (\theta^0, A^0, r^0, \theta^1, A^1, r^1, ..., \theta^t, A^t, r^t, ...).$$
The episode ends when a predetermined number of steps is reached or when the utility of $N_{\theta^t}$ reaches a threshold. 
$\phi$ is updated at the end of each episode based on the rewards collected during the entire episode using the REINFORCE algorithm \cite{policygradient}. 

\textbf{Reward function.} The reward function is constructed with the comprehensive fairness measurement ($\bar{F}$) and comprehensive utility measurement ($\Bar{U}$) defined in Section~\ref{sec3.4}, with the weights adjusted by parameter $\lambda$ ($0 \leq \lambda \leq 1$), as shown in Eq.~\ref{eqNext}.
\begin{equation}
\label{eqNext}
    \mathcal{R}=\lambda \Bar{F} + (1-\lambda)\Bar{U}
\end{equation}

\textbf{Get updated models.}
During the training process, we obtain different parameters $\theta$ and corresponding models $N_{\theta}$. Meanwhile, we record the scores of $N_{\theta}$ on comprehensive fairness measurement and comprehensive utility measurement, record as $\bar{F}_{\theta}$ and $\bar{U}_{\theta}$ respectively. Using $\bar{F}_{\theta}$ and $\bar{U}_{\theta}$ of each model as the horizontal and vertical coordinates respectively, we can obtain their scatters. The set of models that form the convex hull of these points is denoted by $\theta^{*}$, which represents the final models obtained from the training process.

\subsection{Comprehensive Measurement}
\label{sec3.4}
In this section, we present how to take multiple metrics into consideration.
The most straightforward way is to put them together. For metrics with positive values, we can directly add them up. While for those with negative values, we can add up their absolute values. However, different metrics may have different scales, and the value change in metrics with small scales would have little influence on the sum. In other words, the effect of these metrics has been diluted. Therefore, the metrics should be normalized first. In addition, different metrics can have different monotonicity. That is, a larger/smaller value for metrics does not necessarily mean better fairness or utility. Thus, we cannot directly add the metrics up even if they are in the same scale, and the manipulation of monotonicity standardization is necessary. 

As shown in Figure~\ref{fig:metrics}, metrics can be classified into three categories with different monotonicity, including monotonically increasing metrics, monotonically decreasing metrics, and non-monotonic metrics. If a larger value in a metric indicates better fairness or higher utility, then the metric is defined as a monotonically increasing metric (Figure~\ref{fig:metrics} (a)). On the contrary, if a smaller value in a metric indicates better fairness or higher utility, then the metric is defined as a monotonically decreasing metric (Figure~\ref{fig:metrics} (b)). While for those whose values have no monotonic correspondence to fairness or utility, they are defined as non-monotonic metrics (Figure~\ref{fig:metrics} (c)). 

\begin{figure}[htbp]
\centering
\includegraphics[width=0.48\textwidth]{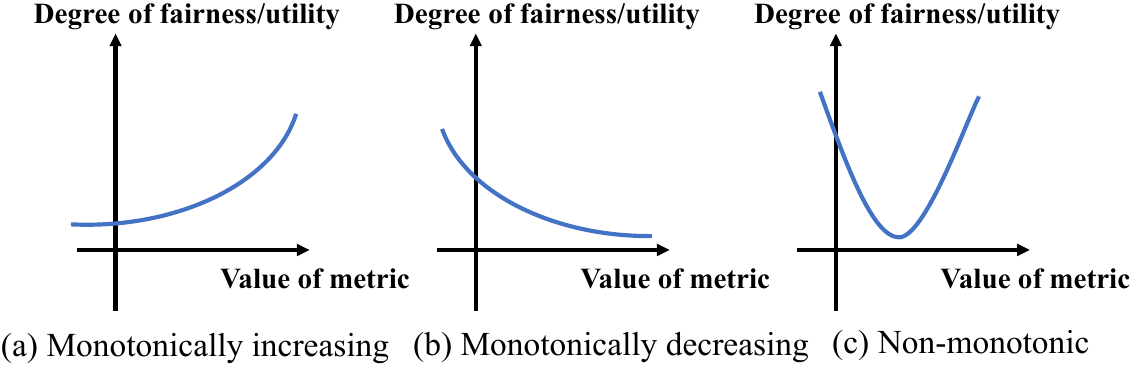}
  \caption{Monotonicity of different metrics}
  \label{fig:metrics}
\end{figure}

Metrics in each category should be processed accordingly so that all the processed metrics share the same monotonicity. In this paper, we process all the metrics to be monotonically increasing with a value range of [0, 1]. The manipulation for metrics with different monotonicity is presented as follows.

\textbf{Monotonically Increasing Metrics}

Denote monotonically increasing metrics by $x_A$, and the normalized metrics can be expressed as Eq.~\ref{eq1},
\begin{equation}
x_A^{\prime}=\frac{x_A-x_{A \min }}{x_{A \max }-x_{A \min }}
\label{eq1}
\end{equation}
where $x_A^{\prime}$ is the normalized metric, $x_{A_{min}}$ is the minimum of the metric, and $x_{A_{max}}$ is the maximum of the metric.

Obviously, there is $x_A^{\prime} \in [0,1]$. Since $x_A$ is already monotonically increasing, the eventual processed metric $X_A$ is then obtained as $X_A=x_A^{\prime}$.

\textit{Example} The accuracy metric is monotonically increasing, as its higher value indicates better utility. 
Let $x_A$ be the value of an accuracy metric, with the maximum value ($x_{A_{max}}$) of 1 and minimum value ($x_{A_{min}}$) of 0.
Then there is $X_A = x^{\prime}_{A} = \frac{x_A - 0}{1-0} = x_{A}$. The processed metric is the same as the original one.

\textbf{Monotonically Decreasing Metrics}

Denote monotonically decreasing metrics by $x_B$, and the normalized metrics can be expressed as Eq.~\ref{eq3},
\begin{equation}
x_B^{\prime}=\frac{x_B-x_{B \min }}{x_{B \max }-x_{B \min }}
\label{eq3}
\end{equation}
where $x_B^{\prime}$ is the normalized metric, $x_{B_{min}}$ is the minimum of the metric, and $x_{B_{max}}$ is the maximum of the metric.

Obviously, there is $x_B'\in[0,1]$. Since $x_B$ is monotonically decreasing, the eventual processed metric $X_B$ should be further calculated with $X_B=1-x_B^{\prime}$.

\textit{Example} As shown in Eq.~\ref{fair1}, SPD is monotonically decreasing, as its lower value indicates better fairness. 
Let $x_B$ be the value of SPD, with the maximum value ($x_{B_{max}}$) of 1 and minimum value ($x_{B_{min}}$) of 0.
Then there is $X_B = 1- x^{\prime}_{B} = 1 - \frac{x_B - 0}{1-0} = 1 - x_B$. The processed metric $SPD^{\prime}$ is presented in Eq.~\ref{eqrspd}.
\begin{equation}
SPD^{\prime} = 1 - |P[\hat{Y}=1|Z=0]-P[\hat{Y}=1|Z=1]|
\label{eqrspd}
\end{equation}

\textbf{Non-monotonic Metrics}

Denote a non-monotonic metric by $x_C$, the minimum of the metric by $x_{C_{min}}$, and the maximum of the metric by $x_{C_{max}}$, then there exists a metric value $x_0 \in (x_{C_{min}},x_{C_{max}})$ that indicates the best fairness or highest utility. Take non-monotonic fairness metrics as an example, $x_C=x_0$ indicates the best fairness, $x_C<x_0$ indicates bias towards privileged group (unprivileged group), while $x_C>x_0$ indicates bias towards unprivileged group (privileged group). In this case, the normalized metrics can be calculated as Eq.~\ref{eq5}.
\begin{equation}
x_C{ }^{\prime}= \begin{cases}\frac{x_C-x_{C \min }}{x_0-x_{C \min }} & x_C \leq x_0 \\ \frac{x_C-x_{C \max }}{x_0-x_{C \max }} & x_C>x_0\end{cases}
\label{eq5}
\end{equation}   

Obviously, there is $x_C^{\prime}\in[0,1]$, and a larger value of $x_C^{\prime}$ indicates better fairness or higher utility. Thus, the eventual processed metric $X_C$ is obtained as $X_C=x_C^{\prime}$.

In particular, when $x_{C_{min}}$ and $x_{C_{max}}$ are symmetrical about (i.e., $x_0=\frac{1}{2}(x_{C_{min}}+x_{C_{max}})$), the processed metric $X_C$ can be readdressed as Eq.~\ref{eq7}.
\begin{equation}
X_C=1-2 \frac{\left|x_C-x_0\right|}{x_{C \max }-x_{C_{\text {min }}}}
\label{eq7}
\end{equation} 

\textit{Example} As shown in Eq.~\ref{fair2}, EOD is non-monotonic,
as a value closer to 0 indicates better fairness. 
Let $x_C$ be the value of EOD, with the maximum value ($x_{C_{max}}$) of 1 and minimum value ($x_{C_{min}}$) of -1. When the metric value is zero ($x_0 = 0$) there is no bias.
Then there is $X_C = 1 - 2 \frac{\left|x_C-x_0\right|}{x_{C_{max}}-x_{C_{min}}}=1 - 2\frac{|x_C-0|}{1-(-1)} = 1 - |x_C|$. The processed metric $EOD^{\prime}$ is presented in Eq.~\ref{eqeod}
\begin{equation}
EOD^{\prime}=1-\left|TPR_u-TPR_p\right| 
\label{eqeod}
\end{equation}  

Denote the mean value of the processed metrics by $\bar{X}$, then there is Eq.~\ref{eq8}.
\begin{equation}
\bar{X}=\frac{1}{N}\left(\sum_{i=1}^I X_{A i}+\sum_{j=1}^J X_{B j}+\sum_{k=1}^K X_{C k}\right)
\label{eq8}
\end{equation}   
\begin{equation}
N=I+J+K
\label{eq9}
\end{equation}   
where $X_{A_i}$ is the value of the $i$-th processed monotonically increasing metrics, $X_{B_j}$ is the value of the $j$-th processed monotonically decreasing metrics, $X_{C_k}$ is the value of the $k$-th processed non-monotonic metrics, $I$ is the number of the processed monotonically increasing metrics employed in an experiment, $J$ is the number of the processed monotonically decreasing metrics employed in the experiment, $K$ is the number of the processed non-monotonic metrics employed in the experiment, and $N$ is the total number of the metrics employed in the experiment and satisfies Eq.~\ref{eq9}.

We define the mean value of the processed fairness metrics (denoted by $\Bar{F}$) as comprehensive fairness measurement, and that of the processed utility metrics (denoted by $\Bar{U}$) as comprehensive utility measurement.
Obviously, $\bar{F}$ has covered multiple fairness metrics, so has $\bar{U}$. 
\subsection{New Metrics}
\label{newmetrics}
Section \ref{sec3.4} has provided an effective method for taking multiple metrics simultaneously into account. But it should be noted that not all metrics are suitable to be combined together. 
We assume that the coupling effects between metrics have a negative effect on the training process, and using metrics that are equivalent but more independent is better. We have demonstrated this through the theoretical analysis in this section and the empirical results in \ref{sec:rq2}.

Take the commonly used EOD, AOD, and ERD as an example, the combination of the three fairness metrics can make the result degenerate into a specific metric.

\begin{theorem}
\label{theorem1}
The combination of EOD, AOD and ERD can always be expressed with a certain one metric of the three.
\end{theorem}

\begin{proof}

Metrics AOD, EOD, and ERD are given by Eqs.~\ref{fair3}-\ref{fair4} respectively, where their terms can be calculated according to Table~\ref{table1}.

All three metrics are non-monotonic metrics and are symmetrical about zero. Thus, they can be treated with Eq.~\ref{eq7}. The processed metrics are presented in Eqs.~\ref{eqeod}, \ref{eq15}-\ref{eq16}.
\begin{equation}
AOD^{\prime}=1-\frac{1}{2}\left|\left(FPR_u-FPR_p\right)+\left(TPR_u-TPR_p\right)\right| 
\label{eq15}
\end{equation}  
\begin{equation}
ERD^{\prime}=1-\frac{1}{2}\left|\left(FPR_u+FNR_u\right)-\left(FPR_P+FNR_P\right)\right|
\label{eq16}
\end{equation}  

Let $a=FPR_u-FPR_p$, $b=FNR_u-FNR_p$, and $c=TPR_u-TPR_p$ ($a,b,c\in [-1,1]$), and the processed metrics can be readdressed as Eqs.~\ref{eq17}-\ref{eq19}.
\begin{equation}
 EOD^{\prime}=1-|c|
\label{eq17}
\end{equation}  
\begin{equation}
 AOD^{\prime}=1-\frac{1}{2}|a+c|
\label{eq18}
\end{equation} 
\begin{equation}
 ERD^{\prime}=1-\frac{1}{2}|a+b|
\label{eq19}
\end{equation}  

According to Table~\ref{table1}, there is Eq.~\ref{eq20}, or in other words, $c=-b$. While variables $a$ and $b$ are mutually independent. Substitute Eq.~\ref{eq20} into Eqs.~\ref{eq17}-\ref{eq19}, and Eqs.~\ref{eq21}-\ref{eq23} could be obtained.
\begin{equation}
b+c=\left(F N R_u+T P R_u\right)-\left(F N R_p+T P R_p\right)=1-1=0
\label{eq20}
\end{equation}  
\begin{equation}
E O D^{\prime}=1-|b|
\label{eq21}
\end{equation}  
\begin{equation}
A O D^{\prime}=1-\frac{1}{2}|a-b|
\label{eq22}
\end{equation}  
\begin{equation}
E R D^{\prime}=1-\frac{1}{2}|a+b|
\label{eq23}
\end{equation}  

The correlations of $EOD^{\prime}$, $AOD^{\prime}$, and $ERD^{\prime}$ are then discussed on different domains of $a$ and $b$ based on Eqs.~\ref{eq21}-\ref{eq23}, as shown in Table~\ref{table2}.

\begin{table}[!ht]
    \caption{
    \centering
    Correlations of $EOD^{\prime}$, $AOD^{\prime}$, and $ERD^{\prime}$ on different domains}
    \begin{tabular}{p{8.5cm}}
        \includegraphics[width=0.45\textwidth]{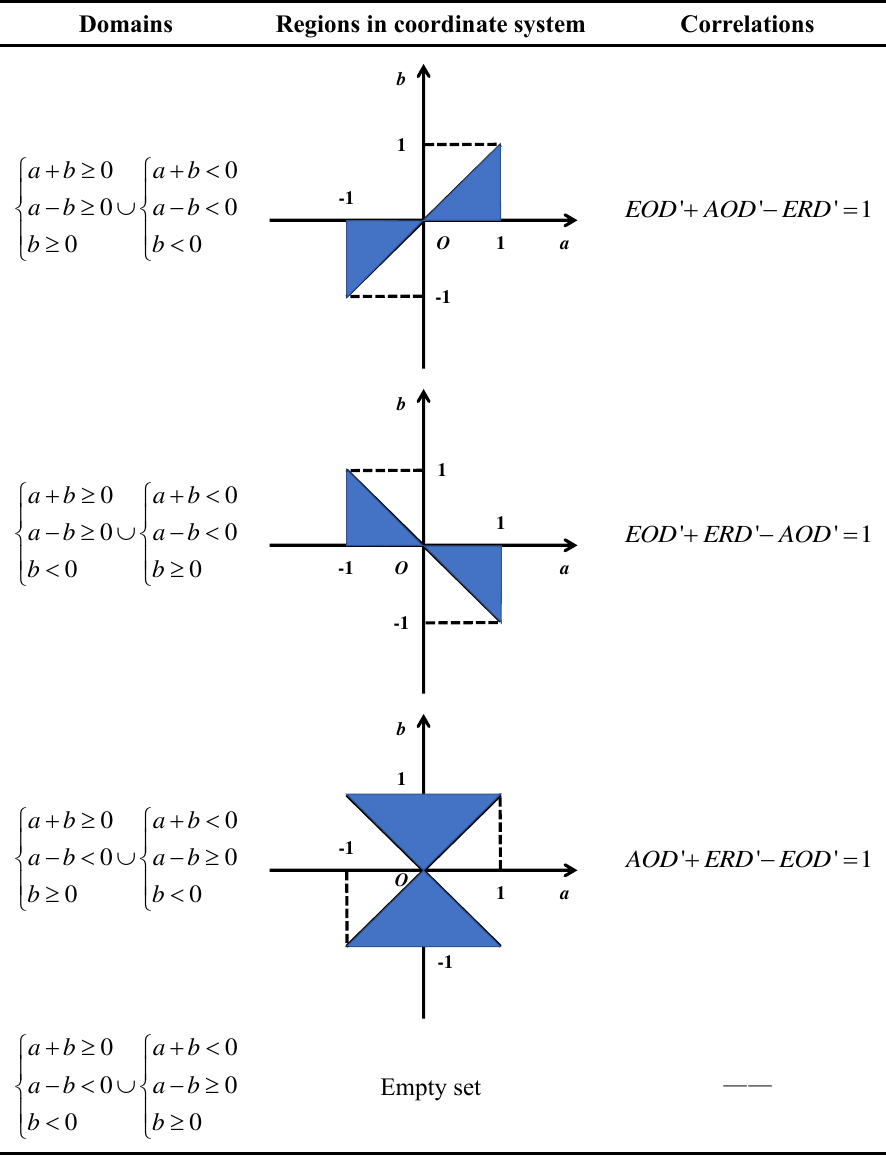}\\
    \textbf{Note}: The shaded areas in the second column are the domains where the correlations hold.
    \end{tabular}
    \label{table2}
\end{table}

Therefore, the sum of $EOD^{\prime}$, $AOD^{\prime}$, and $ERD^{\prime}$ always satisfies one of the three equations of Eqs.~\ref{eq24}-\ref{eq26} depending on different domains of $a$ and $b$. That means the combination of the three metrics can always be expressed with a certain one metric and cannot play a combined role as expected (as expressed by Eq.~\ref{eq27}, in which $X_i$  represents any one of the three metrics).
\begin{equation}
EOD^\prime+AOD^\prime+ERD^\prime=1+2ERD^\prime
\label{eq24}
\end{equation}  
\begin{equation}
EOD^\prime+AOD^\prime+ERD^\prime=1+2AOD^\prime
\label{eq25}
\end{equation}  
\begin{equation}
EOD^\prime+AOD^\prime+ERD^\prime=1+2EOD^\prime
\label{eq26}
\end{equation}  
\begin{equation}
EOD^\prime+AOD^\prime+ERD^\prime=1+2X_i
\label{eq27}
\end{equation}  
\end{proof}

In this case, two new metrics named by $m_a$ and $m_b$ are developed herein to eliminate the negative effects of direct combination, as shown in Eq.~\ref{eq28}-\ref{eq29}. According to the definitions of metric monotonicity, the two metrics are non-monotonic and can be processed as Eq.~\ref{eq30}-\ref{eq31} respectively. Since $M_a$ and $M_b$ are mutually independent and have no negative coupling effects, they are expected to be able to achieve better results compared with the direct combination of $EOD^{\prime}$, $AOD^{\prime}$, and $ERD^{\prime}$.
\begin{equation}
m_a=a=FPR_u-FPR_p
\label{eq28}
\end{equation}  
\begin{equation}
m_b=b=FPR_u-FPR_p
\label{eq29}
\end{equation}  
\begin{equation}
M_a=1-|a|=1-|FPR_u-FPR_p|
\label{eq30}
\end{equation}  
\begin{equation}
M_b=1-|b|=1-|FNR_u-FNR_p|
\label{eq31}
\end{equation} 

So, can we say the combination of $EOD^{\prime}$, $AOD^{\prime}$, and $ERD^{\prime}$ is actually good-for-nothing? The answer is no. Although the combination can always be expressed with a certain one of the three metrics, it cannot be always expressed with a fixed one. In other words, the metrics expressing the combination can be different during an experiment. With variables $a$ and $b$ falling on different domains, as illustrated in Table~\ref{table2}, $X_i$ in Eq.~\ref{eq27} can change from one metric to another. In this sense, we can still say the combination has covered all the three metrics, merely the metrics are covered one by one, rather than simultaneously. Therefore, the combination of $EOD^{\prime}$, $AOD^{\prime}$, and $ERD^{\prime}$ can still outperform any one of the three metrics that acts as an individual, even if it may not match the newly developed combination of $M_a$ and $M_b$.

\section{Experiment Setup}\label{chap:exper}
\subsection{Datasets}
We conduct experiments on 4 commonly used datasets in fairness-related research \cite{uclpri, uclcom, white, ase1, uclase, aseaut}, namely Adult \cite{adult}, Compas \cite{compas}, German \cite{german} and Bank \cite{bank}, in which Adult and Compas have two sensitive attributes. In each task, only one sensitive attribute is considered, leading to 6 tasks in total (e.g., Compas-Sex, Compas-Race, etc.). 
Table~\ref{tabledata} 
provides an overview of these datasets, where Sen. Att., and Fav. Lab. are the abbreviations for sensitive attributes, and favorable labels respectively. Detailed descriptions of these datasets are given below. 
\begin{itemize}
\item \textbf{Adult} aims to predict whether an individual's income exceeds 50 thousand dollars per year based on  personal finance and demographic information.
\item \textbf{Compas} aims to predict whether a criminal offender would recidivate within two years based on criminal history and demographic information of offenders.
\item \textbf{German} aims to predict whether an individual applying for a loan has the ability to repay it based on customer demographics and financial indicators.
\item \textbf{Bank} aims to predict whether a customer would subscribe to a term deposit based on the information about the marketing campaign of a Portuguese bank. 
\end{itemize}

\begin{table}[!ht]
    \centering
    \caption{Dataset Information}
    \begin{tabular}{cccccc}
    \hline
        \textbf{Dataset} & \textbf{Size}  & \textbf{Sen. Att.} & \textbf{Privileged } & \textbf{Fav. Lab.} \\ \hline
        \textbf{Adult} & 48,842  & sex, race & male, white  & income$>$50K\\ 
        \textbf{Compas} & 7,214  & sex, race & female, white  & no recidivism\\ 
        \textbf{German} & 1,000  & sex & male & good credit\\ 
        \textbf{Bank} & 30,488  & age &  age$>$25 & yes\\
       \hline
    \end{tabular}
    \label{tabledata}
\end{table}

\subsection{Bias Mitigation Methods}
We compare \CFU~with 4 existing bias mitigation methods, including Disparate Impact Remover (\DIR) \cite{DI}, Adversarial Debiasing (\ADV) \cite{adv}, Reject Option Classification (\ROC) \cite{roc}, and \GetFair~\cite{getfair}. \DIR, \ADV~and \ROC~are widely used methods \cite{maat, cct31}, and \GetFair~is one of the most recent methods. 
\begin{itemize}
\item  \textbf{DIR} is a pre-processing method based on the disparate impact metric. It modifies the values of the non-sensitive attribute to remove the bias from the training dataset.
\item \textbf{ADV} is an in-processing method based on adversarial techniques. The discriminator learns the difference between sensitive and other attributes to mitigate the bias.
\item \textbf{ROC} is a post-processing method that assigns favorable (unfavorable) labels to the unprivileged (privileged) groups.
\item \textbf{GetFair} is a reinforcement learning-based method. Its agent can output the action to update the parameters of the classifier. The reward function of \GetFair~uses single fairness metric (SPD, EOD or AOD).
\end{itemize}

\subsection{Metrics}
\subsubsection{Utility Metrics} 
We measure the utility of the model for a given task using 3 classical classification metrics, including accuracy (ACC), F1-Score (F1) and area under the ROC curve (AUC). Accuracy calculates the proportion of samples that are correctly classified out of all classification samples. The F1 score is a weighted harmonic mean of precision and recall. 
The ROC curve is formed by the points with false positive rate (FPR) as the horizontal axis and true positive rate (TPR) as the vertical axis.
AUC measures the quality of the predictions across different probability thresholds.

\subsubsection{Fairness Metrics} We measure the fairness of the model using 5 metrics described in Section~\ref{sec:3.1.2}, including DI, SPD, AOD, EOD, and ERD. The definitions are given by Eqs.~\ref{fair5}-\ref{fair4}.

The metrics are processed to be monotonically increasing with the rage of [0, 1] based on the manipulation described in Section~\ref{sec3.4}.
The three utility metrics and DI have already satisfied the requirement. 
SPD, EOD, AOD and ERD are processed to $SPD^{\prime}$, $EOD^{\prime}$, $AOD^{\prime}$, and $ERD^{\prime}$ respectively as shown in Eqs.~\ref{eqrspd}, \ref{eqeod}, \ref{eq15}-\ref{eq16}.

\subsection{Trade-off Measurement}
Fairea \cite{fairea} is a benchmarking tool for measuring the fairness-utility trade-off of bias mitigation methods, and has been applied in related research \cite{maat}. Fairea constructs a trade-off baseline based on the utility and fairness of the original model and a set of pseudo-models created by randomly mutating model predictions. 
As illustrated in Fig.~\ref{fig:fairea}, the baseline is constructed by plotting the utility of the Fairea models on the horizontal axis and the fairness on the vertical axis.
The plane is divided into 5 regions with different trade-off effectiveness levels based on the baseline,
namely \textit{win-win}, \textit{good}, \textit{inverted}, \textit{bad}, and \textit{lose-lose}. 
The point of a model falls in \winwin~if its performance and fairness are both better than the baseline. If it is worse on either, it belongs to \loselose. 
If it has better utility but worse fairness, it belongs to \inverted. 
If fairness is improved with a utility loss, there are two possibilities: 
If the trade-off of the model is better than the baseline, it belongs to \good; otherwise, it belongs to \bad.
The bias mitigation method can achieve a better fairness-utility trade-off when its processed model falls within the \winwin~more frequently.

$Example$ If 150 points are obtained from the processed model with 75 falling in \winwin~and 15 falling in \good, then the proportion of \winwin~is 50\% and \good~is 10\%. The same manipulation is performed for each combination of fairness metric (as horizontal axis) and utility metric (as vertical axis), and the average proportions of each region are calculated and taken as the final result.
\begin{figure}[htbp]
  \centering
  \includegraphics[width=0.25\textwidth]{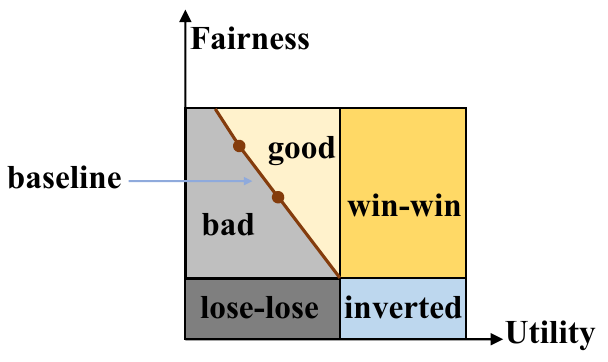}
  \caption{Trade-off Baseline Fairea}
  \label{fig:fairea}
\end{figure}


\subsection{Implementation}
The default reward function of \CFU~is $0.5(M_a+M_b)+0.5AUC$.
Four datasets and three methods including \DIR, \ADV~and \ROC~are used by directly invoking off-the-shelf APIs provided by IBM AIF360 toolkit \cite{aif360}.
We apply \GetFair~method based on the code released by its authors \cite{getfaircode}, and adopt Fairea based on the code of related research MAAT \cite{maatcode}.
For each task, we train the original models using 3 ML algorithms including Neural Network (NN), Logistic Regression (LR), and Support Vector Machine (SVM), which are widely adopted in fairness literature \cite{cct31, ruler, maat, cct6, sun3, sun4}. 
Following previous work \cite{cct31}, our NN contains five hidden layers, each consisting of 64, 32, 16, 8, and 4 units respectively. For SVM and LR \cite{maat}, we employed the default configuration provided by the scikit-learn library \cite{scikitlearn}.

\subsection{Research Question}
We conduct experiments to answer the following questions:

\textbf{RQ1: How effective and efficient is \CFU~in improving fairness-utility trade-off?} 
This research question evaluates the performance of \CFU~in improving 15 fairness-utility trade-offs of 3 ML algorithm classifiers on 6 tasks.
We illustrate the fairness and utility of processed models and analyze their trade-off effectiveness levels based on the Fairea benchmarking tool.

\textbf{RQ2: What about the effectiveness of the newly proposed fairness metrics $M_a$ and $M_b$?} To address this question, ablation experiments are conducted on different combinations of fairness and utility metrics, and Fairea is employed to evaluate their effectiveness. 


\section{Results and Discussion}\label{chap:resul}
\subsection{RQ1: Effectiveness and Efficiency}\label{sec:rq1}

To demonstrate the effectiveness of \CFU, we conduct experiments on \CFU~and 4 existing methods across 3 algorithms (LR, SVM, NN) and 6 tasks.
Each experiment is repeated 10 times, and a total of $6 \times 3 \times 6 \times 10 = 1080$ experiments are performed. 

The results are analyzed in 3 steps. At first, the fairness-utility trade-off of bias mitigation methods on all algorithms and tasks is evaluated by Fairea. Afterward, we present the results on different ML algorithms and different tasks. Finally, the fairness and utility of the original model and models processed by different methods on a specific algorithm and task are presented. 

\textbf{Effectiveness.} As shown in Figure~\ref{fig:rq1}, \CFU~(54.8\%) have achieved the highest proportion for \winwin~cases. \ROC~(8.9\%), \GetFair~(33\%), \DIR~(21.5\%), and \ADV~(5.9\%)~have decreased by 45.9\%, 21.8\%, 33.3\% and 48.9\% in proportions respectively, with an average of 37.5\% compared with \CFU. This shows that \CFU~can effectively improve the fairness-utility trade-off for ML classifiers.

\begin{figure}[htbp]
  \centering
  \includegraphics[width=0.47\textwidth]{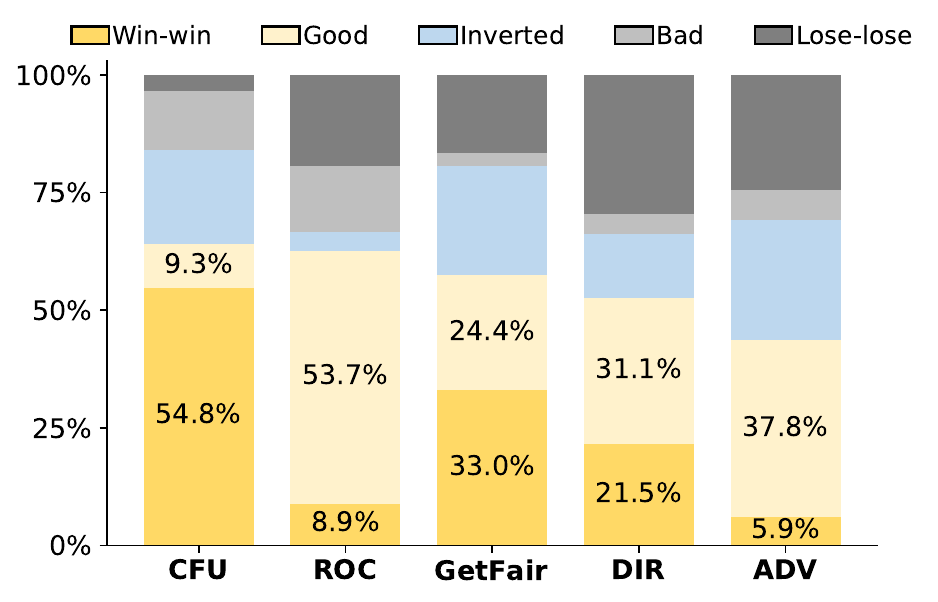}
  \caption{(RQ1) Distributions of trade-off effectiveness levels for different bias mitigation methods.
}
  \label{fig:rq1}
\end{figure}

\begin{figure}
  \centering
  \begin{subfigure}[b]{0.47\textwidth}
    \includegraphics[width=\textwidth]{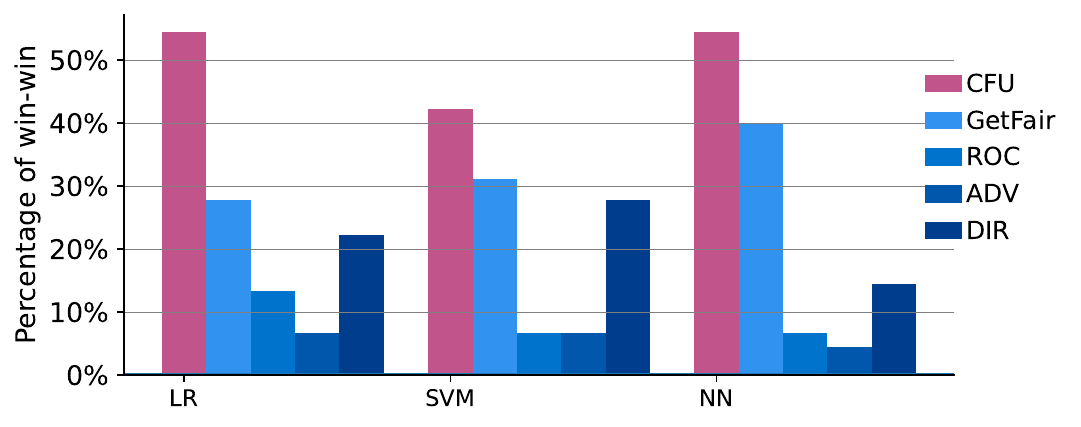}
    \caption{Different ML algorithms.}
    \label{fig:rq2-model}
  \end{subfigure}

  \begin{subfigure}[b]{0.47\textwidth}
    \includegraphics[width=\textwidth]{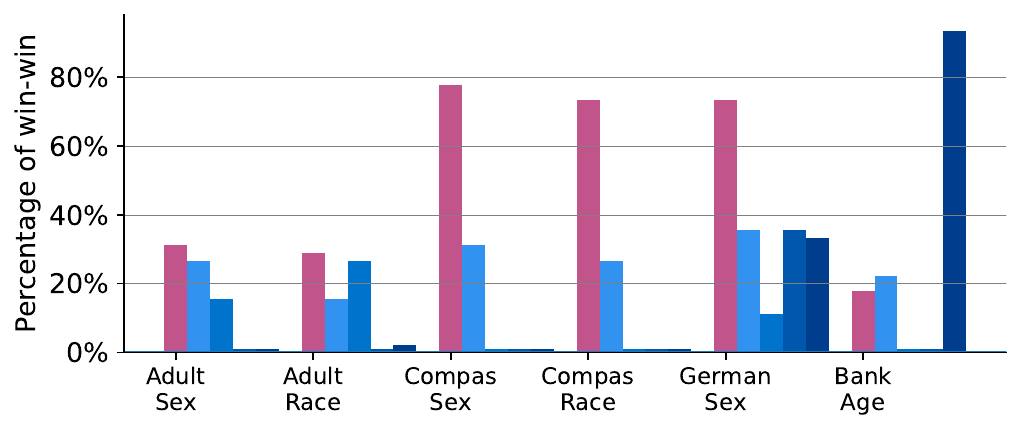}
    \caption{Different tasks.}
    \label{fig:rq2-data}
  \end{subfigure}

  \caption{(RQ1) Trade-off effectiveness of different bias mitigation methods with respect to 3 different algorithms and 6 different tasks.
  }
  \label{fig:rq2}
\end{figure}
\begin{figure*}[htbp]
  \centering
  \includegraphics[width=1.0\textwidth]{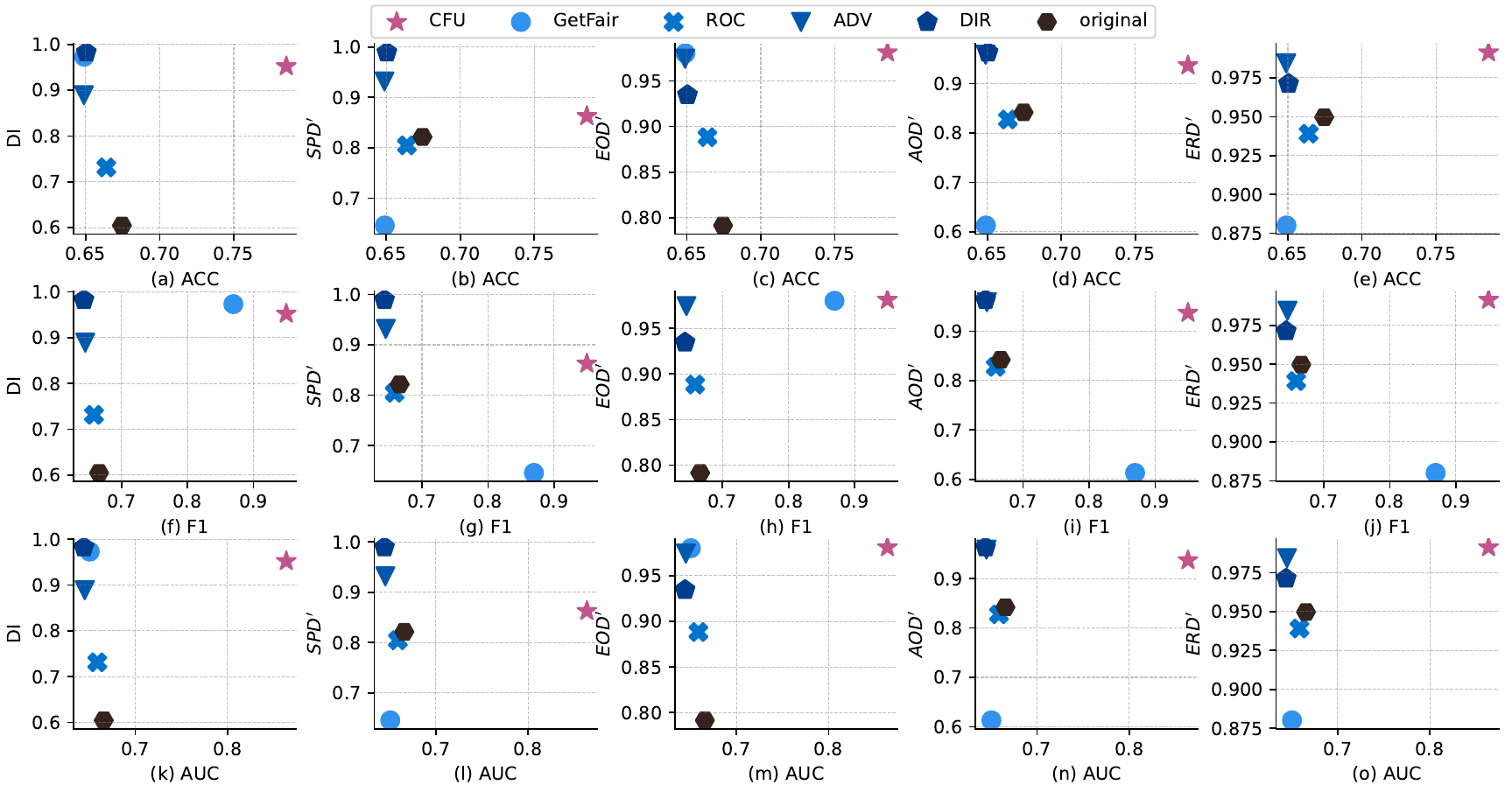}
  \caption{(RQ1) 
The fairness-utility trade-off of different methods on LR and Compas_Race. 
The original points correspond to the original LR model, while the points of bias mitigation methods correspond to the models processed by these methods. 
It can be seen that the fairness and utility of the model processed by \CFU~have been improved across various metrics.}
  \label{fig:rq1.2}
\end{figure*}
\textbf{Generality and Stability.} The results on different ML algorithms are presented in Figure~\ref{fig:rq2} (a). It can be found that \CFU~achieves the highest proportion of cases falling in \winwin~for all algorithms, with 54\% for LR, 42\% for SVM, and 54\% for NN. The results on different tasks are illustrated in Figure~\ref{fig:rq2} (b). For \CFU, it outperforms all other methods in most tasks, except for Bank-Age. In addition, the stability of \CFU~is also investigated. Compared with other methods, \CFU~has exhibited much higher stability. The proportions of cases falling in \winwin~for \CFU~are all higher than 17\% regardless of the tasks. While for other methods, much greater fluctuations in such proportion can be observed. For instance, \ROC~has more than 15\% cases falling in \winwin~on Adult-Sex and Adult-Race, but such proportion drops to 0\% on Bank-Age. The case is worse for \DIR, which has 100\% cases falling in \winwin~on Bank-Age but 0\% on Adult-Sex, Compas-Sex, and Compas-Race. In contrast, \CFU~has at least 31\% cases falling in \winwin~on these 3 tasks while still reaching 18\% on Bank-Age. Therefore, \CFU~is effective in improving the fairness-utility trade-off of various ML algorithms on most tasks and is more stable than existing methods.

\textbf{Efficiency.} 
The scatter diagram of the original model and models processed by 5 bias mitigation methods are plotted in Figure~\ref{fig:rq1.2},
with the utility value of the model as abscissa and the fairness value as ordinate. 
15 subgraphs have a different coordinate axis, which is a combination of utility metric and fairness metric. Due to limited space, only the results for the LR algorithm on Compas-Race task are presented, and those for other tasks are similar. Almost all the points for \CFU~are located at the upper right of the points for the original model, indicating that the model processed by \CFU~has better fairness and utility and can achieve a better fairness-utility trade-off. Moreover, \CFU~works well for all fairness metrics. The model processed by \CFU~reaches ((c), (h), (m), (e), (j), and (o) subgraphs) or close to (other subgraphs) the highest fairness scores of all models. In contrast, most points for other bias mitigation methods are located in the upper-left of the points for the original model, indicating that they have sacrificed utility while improving fairness. 
In addition, they also restricted to improving only one specific fairness notion.
For instance, \DIR~is good at improving statistical parity, with the highest score on both DI and SPD metrics, while \GetFair~and \ADV~work well at equal opportunity and equalized odds respectively. 
This suggests that \CFU~can efficiently improve the fairness-utility trade-off across multiple metrics.

\begin{answerbox}
\textbf{Answer to RQ1:} 
\CFU~is superior and has outperformed existing methods in improving fairness-utility trade-off across multiple metrics.
In addition, it has superior generality and stability over existing methods on various ML algorithms and tasks.
\end{answerbox}

\begin{table}[]
\caption{\centering(RQ2) 
Effectiveness of \CFU~when considering different metrics in reward function.}
\renewcommand{\arraystretch}{1.25}
		\centering
  \LARGE
		\resizebox{0.5\textwidth}{!}{
		\begin{tabular}{llllllllll|c}
			\hline
			\multicolumn{10}{c|}{\textbf{Metrics used in reward function}} & \multirow{2}{*}{\winwin}\\
			\cline{1-10}
			$M_a$ &$M_b$ &DI &SPD &EOD &AOD &ERD &ACC &F1& AUC & \\
      	\hline
			\hline 
                + &+ & & & & & & & &+ & 54.81(100.0\%)\\
 &+ & & & & & & & &+ & 54.44(99.32\%)\\
+ & & & & & & &+ & &+ & 53.7(97.97\%)\\
 &+ & & & & & &+ & &+ & 50.37(91.89\%)\\
+ & & & & & & & & &+ & 50.37(91.89\%)\\
+ &+ & & & & & &+ & & & 46.3(84.46\%)\\
 &+ & & & & & &+ & & & 45.93(83.78\%)\\
+ &+ & & & & & &+ & &+ & 44.44(81.08\%)\\
+ & & & & & & &+ & & & 41.11(75.0\%)\\
 & &+ &+ & & & & & &+ & 40.0(72.97\%)\\
 & &+ &+ &+ &+ &+ &+ & &+ & 37.04(67.57\%)\\
 & & &+ &+ &+ &+ &+ & &+ & 36.3(66.22\%)\\
 & & &+ &+ &+ &+ & & &+ & 35.56(64.86\%)\\
 & & &+ & & & &+ & &+ & 34.44(62.84\%)\\
 & &+ & &+ &+ &+ &+ & &+ & 34.44(62.84\%)\\
 & &+ &+ & & & &+ & &+ & 34.07(62.16\%)\\
 & &+ & &+ &+ &+ & & &+ & 34.07(62.16\%)\\
 & & & & &+ & & & &+ & 33.7(61.49\%)\\
 & & &+ & & & & & &+ & 33.7(61.49\%)\\
 & & & &+ &+ &+ &+ & &+ & 33.7(61.49\%)\\
 & &+ &+ &+ &+ &+ & & &+ & 33.33(60.81\%)\\
 & & &+ & & & &+ & & & 32.96(60.14\%)\\
 & &+ & & & & & & &+ & 32.22(58.78\%)\\
 & & & & & &+ & & &+ & 31.85(58.11\%)\\
 & & &+ &+ &+ &+ &+ & & & 30.74(56.08\%)\\
 & &+ & & & & &+ & &+ & 30.37(55.41\%)\\
 & & & &+ & & & & &+ & 30.37(55.41\%)\\
 & &+ &+ & & & &+ & & & 30.0(54.73\%)\\
 & &+ & & & & &+ & & & 30.0(54.73\%)\\
 & & & & &+ & &+ & & & 29.26(53.38\%)\\
 & &+ & &+ &+ &+ &+ & & & 29.26(53.38\%)\\
 & & & &+ & & &+ & & & 26.3(47.97\%)\\
 & &+ &+ &+ &+ &+ &+ & & & 25.56(46.62\%)\\
 & & & & & &+ &+ & & & 24.81(45.27\%)\\
 & &+ & & & & & &+ & & 23.7(43.24\%)\\
 & & & &+ & & & &+ & & 23.7(43.24\%)\\
 & & &+ & & & & &+ & & 23.7(43.24\%)\\
 & & & & &+ & & &+ & & 18.15(33.11\%)\\
 & & & & & &+ & &+ & & 17.41(31.76\%)\\
			\hline
   \addlinespace[1.5ex]
   \multicolumn{11}{p{22cm}}{\Huge\textbf{Note:} + means that the metric is used.
   In the last column, values outside the brackets indicate the percentage of \winwin~level, while those in the brackets indicate the ratio of the outside value to the best one. The table is sorted in descending order of the percentage.}
		\end{tabular}}
		\label{tab:ablation}
	\end{table}

\subsection{RQ2: Effectiveness of New Metrics}\label{sec:rq2}
To demonstrate the effectiveness of newly developed fairness metrics $M_a$ and $M_b$, ablation experiments are conducted herein. Overall, we test numerous reward functions, which are constructed with different metrics. The results are evaluated by Fairea and the proportions of \winwin~cases are recorded, as presented in Table~\ref{tab:ablation}. 

We use 10 metrics, including our proposed 2 fairness metrics, 5 existing fairness metrics (SPD, DI, EOD, AOD, and ERD), and 3 utility metrics (ACC, F1, and AUC). The ablation experiments consist of 3 groups. 
Firstly, we employ 1 existing fairness metric paired with 1 utility metric to form a total of 15 combinations ($3 \times 5$). 
Secondly, we bundle EOD, AOD, and ERD fairness metrics together (select all or none) and name them as Triplet. 
Using 3 fairness metrics (SPD, DI, Triplet) and 2 utility metrics (ACC, AUC), we explore a total of 15 combinations ($(C_3^1 + C_3^2 + C_3^3) \times (C_2^1+C_2^2)-C_3^1*C_2^1$).
Finally, we combine our two new metrics with 2 utility metrics (ACC, AUC) to form a total of 9 combinations ($(C_2^1+C_2^2) \times (C_2^1+C_2^2)$). 
We use each combination to construct \CFU~and do 10 repeated experiments on 3 algorithms and 6 tasks, resulting in a total of 7020 experiments ($(15 + 15 + 9) \times 3 \times 6 \times 10$). 


\textbf{The efficiency of $M_a$ and $M_b$.} The highest proportion of \winwin~cases (54.81\%) is achieved when both $M_a$ and $M_b$ are used as fairness metrics and AUC as performance metric (the default setting of reward function), while the combinations without $M_a$ and $M_b$ seem to be not that effective: the \winwin~cases for combinations including $M_a$ or $M_b$ are all higher than those without them. The highest proportion of \winwin~cases that can be achieved without $M_a$ and $M_b$ is 40\%, around 72.75\% of default \CFU. On average, the cases falling in \winwin~for combinations using newly proposed metrics have been improved by 61.58\% compared with the combinations without using it. Therefore, it is concluded that $M_a$ and $M_b$ are superior in improving the trade-off between fairness and utility.

\textbf{Characteristics of existing metrics.} It can be found that the combinations containing F1 have all experienced a decrease in proportions of \winwin~cases. In contrast, those containing AUC can achieve higher proportions for \winwin~cases. These findings suggest that F1 is ineffective in \CFU, while AUC is an excellent utility metric. Additionally, Using EOD, AOD, and ERD simultaneously can still perform well in \winwin~cases even though it can not match default \CFU.

\textbf{The superiority of \CFU~comes from the comprehensive measurement and new metrics.} In RQ1, we can see that \GetFair~cannot compete with \CFU~when it comes to \winwin~cases. In the ablation experiments, \GetFair~has three kinds of combinations, including SPD+ACC, EOD+ACC, and AOD+ACC. The highest proportion of \winwin~cases that are achieved by these combinations is 32.96\%, only 60.14\% of default \CFU. Since the only difference between \CFU~and \GetFair~lies in the reward function, which has involved the comprehensive measurement and newly proposed metrics for \CFU, it is reasonable to believe that the comprehensive measurement and new metrics do have played a vital role.


\begin{answerbox}
\textbf{Answer to RQ2:} 
$M_a$ and $M_b$ are superior in improving the trade-off between fairness and utility. 
The best combination of \CFU~consists of $M_a$, $M_b$ and AUC, which has outperformed all other metric combinations.
\end{answerbox}

\section{Conclusion}\label{chap:concl}
In this paper, \CFU~framework is developed based on our comprehensive measurement.
New metrics are proposed for \CFU~so that a better trade-off between fairness and utility could be achieved.
We conduct extensive experiments on \CFU~and compare it with the existing bias mitigation techniques.
The conclusions are summarized as follows. 
\CFU~has realized  37.5\% improvement over the state-of-the-art method across all cases we evaluated. Moreover, it has a superior generality that outperforms other methods on all models and tasks.
Our new metrics can significantly improve the effect of fairness-utility trade-off, and an average increase of 61.58\% has been observed compared with combinations of existing metrics.
It should be noted that even though we focus on classification models, our comprehensive measurement and \CFU~framework could also be used in the improvement of all ML tasks for better fairness-utility trade-off.

\newpage

\normalem
\bibliographystyle{IEEEtran}
\bibliography{IEEEfull}

\end{document}